\documentclass[10pt,twocolumn,letterpaper]{article}
\pdfoutput=1

\usepackage{iccv}
\usepackage{times}
\usepackage{epsfig}
\usepackage{graphicx}
\usepackage{amsmath}
\usepackage{amssymb}

\usepackage{algorithm}
\usepackage[noend]{algpseudocode}
\usepackage{xfrac} % Allows slanted fractions
\usepackage{tabularx} % Automatically calculates column widths with text wrapping
\usepackage{amsthm}
\usepackage{accents} % To typeset lower bounds
\usepackage{mathtools} % For dcases

\makeatletter\let\captiontemp\@makecaption\makeatother
\usepackage[font=footnotesize]{subcaption}
\makeatletter\let\@makecaption\captiontemp\makeatother

\usepackage[pagebackref=true,breaklinks=true,letterpaper=true,colorlinks,bookmarks=false]{hyperref}

\iccvfinalcopy % *** Uncomment this line for the final submission

\def\iccvPaperID{2218} % *** Enter the ICCV Paper ID here

\ificcvfinal\fi

\graphicspath{{./images/}}

\newcommand{\bR}{\mathtt{R}}
\newcommand{\bt}{\mathbf{t}}
\newcommand{\br}{\mathbf{r}}
\newcommand{\bx}{\mathbf{x}}

\newcommand{\bbf}{\mathbf{f}}
\newcommand{\bp}{\mathbf{p}}
\newcommand{\be}{\mathbf{e}}
\newcommand{\bO}{\mathbf{O}}

\newcommand{\sC}{\mathcal{C}}
\newcommand{\sV}{\mathcal{V}}
\newcommand{\sS}{\mathcal{S}}
\newcommand{\sSk}{\mathcal{S}\hspace{-1pt}k}
\newcommand{\sP}{\mathcal{P}}
\newcommand{\sF}{\mathcal{F}}

\newcommand{\indfn}{\mathbf{1}}
\newcommand{\defeq}{\triangleq}
\newcommand{\transpose}{^{\intercal}}

\newcommand{\ubar}[1]{\underaccent{\bar}{#1}}

\makeatletter
\let\OldStatex\Statex
\renewcommand{\Statex}[1][3]{
  \setlength\@tempdima{\algorithmicindent}
  \OldStatex\hskip\dimexpr#1\@tempdima\relax}
\makeatother

\newtheorem{theorem}{Theorem}
\newtheorem{lemma}{Lemma}

\begin{document}

%%%%%%%%%%%%%%%%%%%%%%%%%%%%%%%%%%%%%%%%%%%%%%%%%%%%%%%%%%
\title{Globally-Optimal Inlier Set Maximisation for\\Simultaneous Camera Pose and Feature Correspondence}
\author{Dylan Campbell\textsuperscript{1,2}, Lars Petersson\textsuperscript{1,2}, Laurent Kneip\textsuperscript{1} and Hongdong Li\textsuperscript{1}\\
\textsuperscript{1}Australian National University%
\thanks{\tiny This research is supported by an Australian Government Research Training Program (RTP) Scholarship.}
\,\,\,\,\,\,\,\,\,\,\,\textsuperscript{2}Data61 -- CSIRO%
\\
{\tt\small \{dylan.campbell,lars.petersson,laurent.kneip,hongdong.li\}@anu.edu.au}
}

\maketitle

%%%%%%%%%%%%%%%%%%%%%%%%%%%%%%%%%%%%%%%%%%%%%%%%%%%%%%%%%%
\begin{abstract}
\vspace{-6pt}
Estimating the 6-DoF pose of a camera from a single image relative to a pre-computed 3D point-set is an important task for many computer vision applications. Perspective-$n$-Point (P$n$P) solvers are routinely used for camera pose estimation, provided that a good quality set of 2D--3D feature correspondences are known beforehand. However, finding optimal correspondences between 2D key-points and a 3D point-set is non-trivial, especially when only geometric (position) information is known. Existing approaches to the simultaneous pose and correspondence problem use \emph{local optimisation}, and are therefore unlikely to find the optimal solution without a good pose initialisation, or introduce restrictive assumptions. Since a large proportion of outliers are common for this problem, we instead propose a globally-optimal inlier set cardinality maximisation approach which jointly estimates optimal camera pose and optimal correspondences. Our approach employs branch-and-bound to search the 6D space of camera poses, guaranteeing global optimality without requiring a pose prior. The geometry of $SE(3)$ is used to find novel upper and lower bounds for the number of inliers and local optimisation is integrated to accelerate convergence. The evaluation empirically supports the optimality proof and shows that the method performs much more robustly than existing approaches, including on a large-scale outdoor data-set.

\vspace{-12pt}
\end{abstract}

%%%%%%%%%%%%%%%%%%%%%%%%%%%%%%%%%%%%%%%%%%%%%%%%%%%%%%%%%%
\section{Introduction}
\label{sec:introduction}

\begin{figure}[!t]
\centering
\vspace{-4pt}
\begin{subfigure}[]{\columnwidth}
\includegraphics[trim=25pt 15pt 35pt 25pt, clip=true, scale=.65]{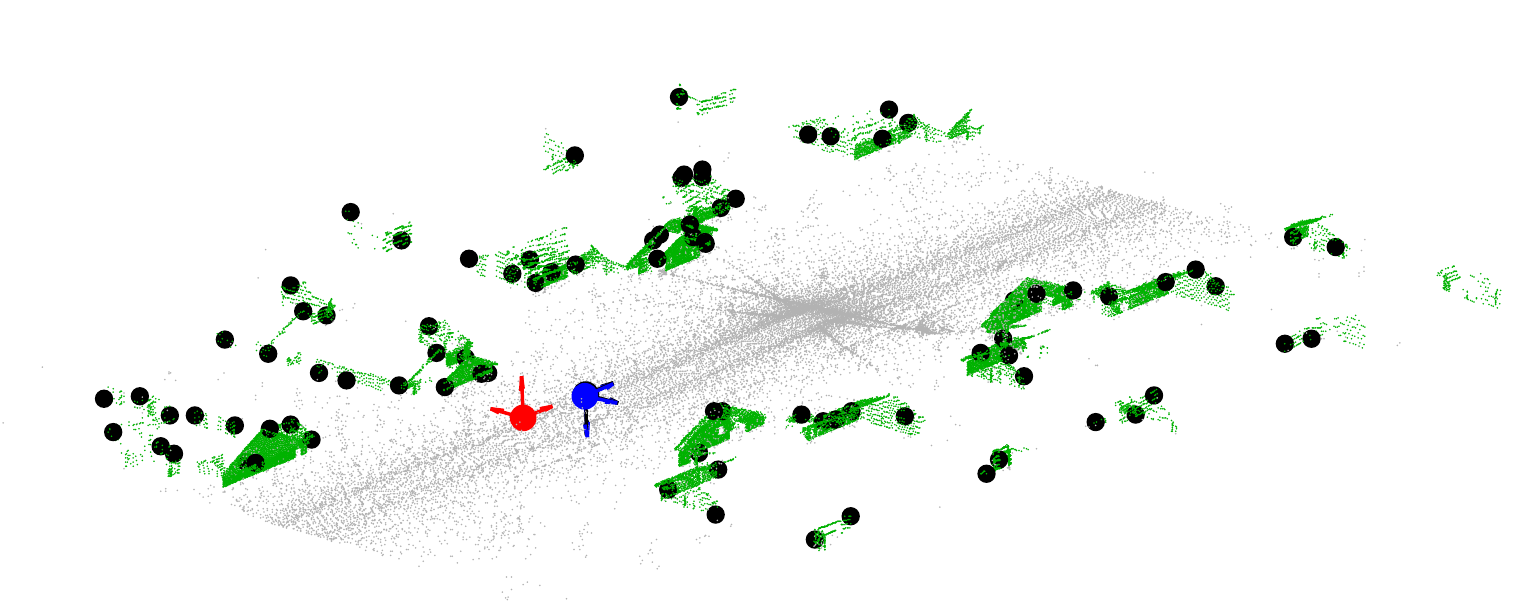}
\caption{3D point-set (grey and green), 3D features (black dots) and ground-truth (black), RANSAC (red) and our (blue) camera poses. The ground-truth and our camera poses coincide, whereas the RANSAC pose has a translation offset and a $180^{\circ}$ rotation offset. Best viewed in colour.}
\label{fig:results_2d3d_3d}
\end{subfigure}
\vfill
\begin{subfigure}[]{\columnwidth}
\includegraphics[trim=156pt 330pt 255pt 196pt, clip=true, scale=.31]{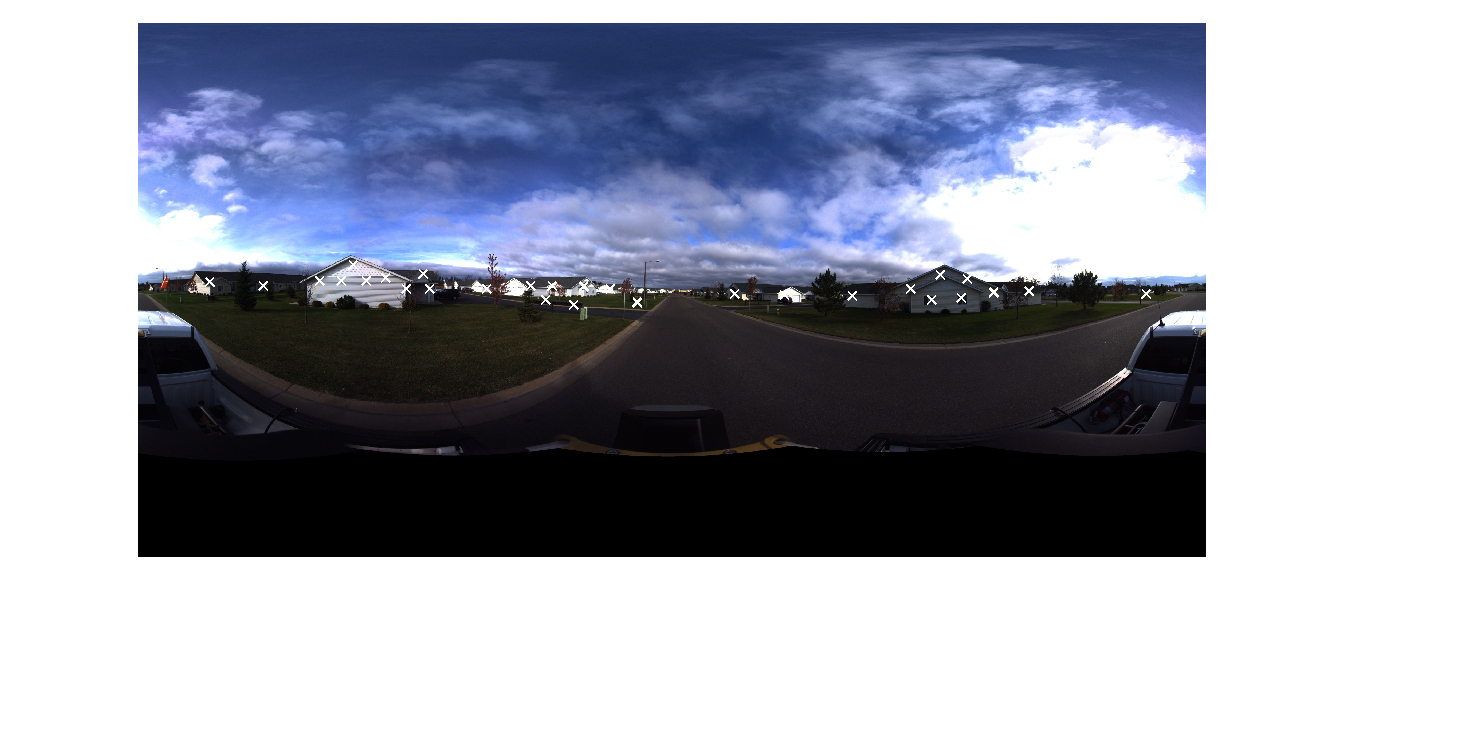}
\vfill
\includegraphics[trim=156pt 330pt 255pt 196pt, clip=true, scale=.31]{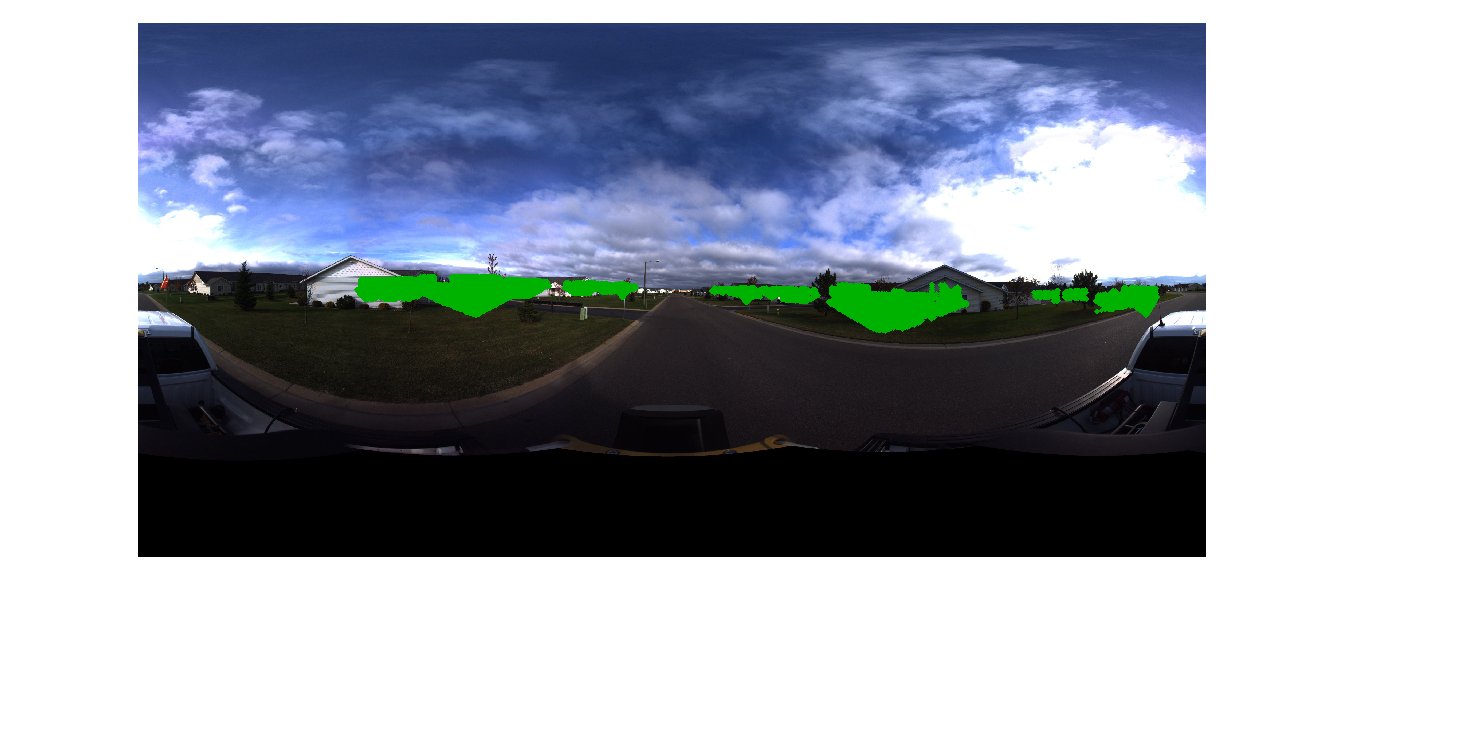}
\vfill
\includegraphics[trim=156pt 330pt 255pt 196pt, clip=true, scale=.31]{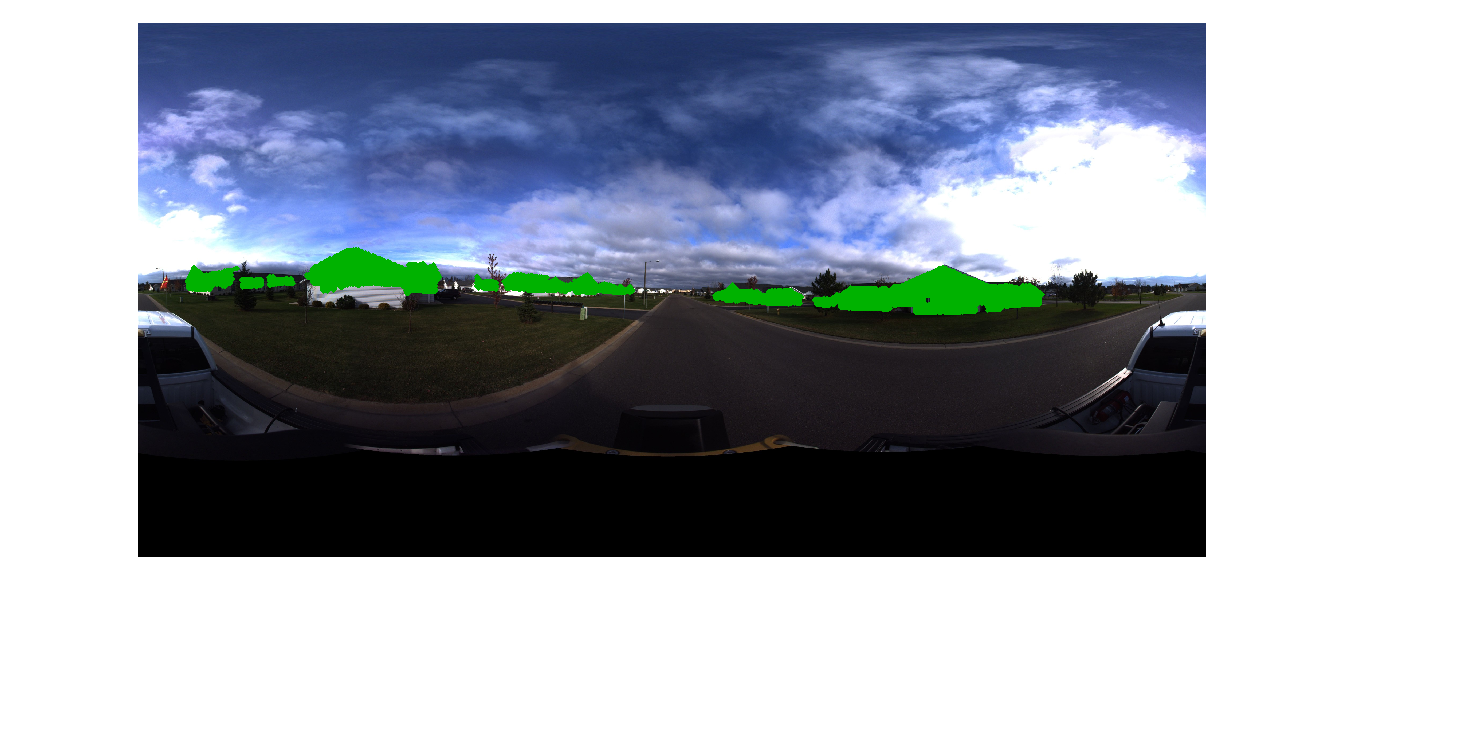}
\caption{Panoramic photograph and extracted 2D features (top), building points projected onto the image using the RANSAC camera pose (middle) and building points projected using our camera pose (bottom).}
\label{fig:results_2d3d_2d}
\end{subfigure}
\caption{Estimating the pose of a calibrated camera from a single image within a large-scale, unorganised 3D point-set captured by vehicle-mounted laser scanner. Our method solves the absolute pose problem while simultaneously finding feature correspondences, using a globally-optimal branch-and-bound approach with tight novel bounds on the cardinality of the inlier set.}
\label{fig:results_2d3d}
\vspace{-12pt}
\end{figure}

Estimating the pose of a calibrated camera given a set of 2D points in the camera frame and a set of 3D points in the world frame, as shown in Figure~\ref{fig:results_2d3d}, is a fundamental part of the general 2D--3D registration problem of aligning an image with a 3D scene or model. When correspondences are known, this becomes the Perspective-$n$-Point (P$n$P) problem for which many solutions exist \cite{haralick1994review,lepetit2009epnp,kneip2011novel,hesch2011direct,kneip2014upnp}. Applications include camera localisation and tracking \cite{fischler1981random, noll2011markerless, kneip2015sdicp}, augmented reality \cite{marchand2016pose}, motion segmentation \cite{olson2001general} and object recognition \cite{huttenlocher1990recognizing,mundy2006object,aubry2014seeing}.

While hypothesise-and-test frameworks like RANSAC \cite{fischler1981random} can mitigate the sensitivity of P$n$P solvers to outliers in the correspondence set, few approaches are able to handle the case where 2D--3D correspondences are not known in advance. Unknown correspondences arise in many circumstances, including the general case of aligning an image with a textureless 3D point-set or CAD model. While feature extraction techniques provide a relatively robust and reproducible way to detect interest points such as edges or corners within each modality, finding correspondences across the two modalities is much more challenging. Even when the point-set has sufficient visual information associated with it, such as colour or SIFT features \cite{lowe2004distinctive}, repetitive features, occlusions and perspective distortion make the correspondence problem non-trivial. Moreover, appearance and thus visual features may change significantly between viewpoints, lighting conditions, weather and seasons, whereas scene geometry is often less affected. When re-localising a camera in a previously mapped environment or bootstrapping a tracking algorithm, we contend that geometry is often more reliable. Therefore, there is a need for methods that solve for both pose and correspondences.

Efficient local optimisation algorithms for solving this joint problem have been proposed \cite{david2004softposit,moreno2008pose}. However, they require a pose prior, search only for local optima and do not provide an optimality guarantee, yielding erroneous pose estimates without a reliable means of detecting failure. Hypothesise-and-test approaches such as RANSAC \cite{fischler1981random}, when applied to the correspondence-free problem \cite{grimson1990object}, are global methods that are not reliant on pose priors but quickly become computationally intractable as the number of points and outliers increase and do not provide an optimality guarantee. More recently, a global and $\epsilon$-suboptimal method has been proposed \cite{brown2015globally}, which uses a branch-and-bound approach to find a camera pose whose trimmed geometric error is within $\epsilon$ of the global minimum.

This work is the first to propose a global and optimal inlier set cardinality maximisation solution to the simultaneous pose and correspondence problem. The approach employs the branch-and-bound framework to guarantee global optimality without requiring a pose prior, ensuring that it is not susceptible to local optima. We use a parametrisation of $SE(3)$ space that facilitates branching and derive novel bounds on the objective function. In addition, we also apply local optimisation whenever the algorithm finds a better transformation, to accelerate convergence without voiding the optimality guarantee. Cardinality maximisation allows an exact optimiser to be found, unlike the $\epsilon$-suboptimality inherent to the continuous objective function used in \cite{brown2015globally}. More critically, cardinality maximisation is inherently robust to 2D and 3D outliers, while avoiding the problems associated with trimming. The latter requires the user to specify the inlier fraction, which can rarely be known and is less intuitive to select than a geometrically meaningful inlier threshold. If the inlier fraction is over- or under-estimated, this approach may converge to the wrong pose, without a means to detect failure. Figure \ref{fig:trimming} demonstrates how the global optimum of a trimmed objective function, as used by \cite{brown2015globally,yang2016goicp}, may not occur at the true pose, a problem that is exacerbated when the inlier fraction is guessed incorrectly.

\begin{figure}[!t]
\centering
\includegraphics{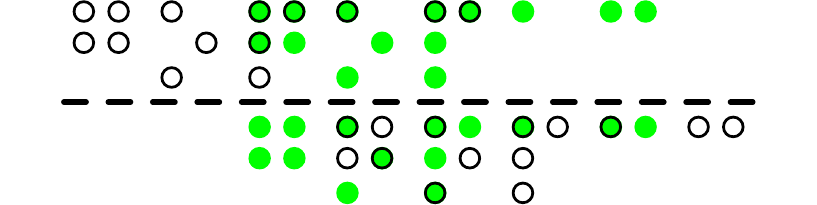}
\caption{Two zero-error but incorrect 1D alignments of 2 point-sets with 8 trimmed `outliers'. With noise, the global optimum of a trimmed objective function may not occur at the true pose, particularly if an incorrect trimming fraction is selected. The problem is exacerbated with higher dimensions and degrees of freedom.}
\label{fig:trimming}
\vspace{-12pt}
\end{figure}

%%%%%%%%%%%%%%%%%%%%%%%%%%%%%%%%%%%%%%%%%%%%%%%%%%%%%%%%%%%%%%%%%%%%%%%%%%%%%%%%%%%%%%%%%%%%%%%
\section{Related Work}
\label{sec:related_work}

A large body of work exists for solving the 2D--3D registration problem when correspondences are provided. When the correspondences are known perfectly, Perspective-$n$-Point (P$n$P) solvers \cite{haralick1994review,lepetit2009epnp,kneip2011novel,hesch2011direct,kneip2014upnp} are able to estimate the pose of a camera given a set of noisy image points and their corresponding 3D points. When outliers are present in the correspondence set, the RANSAC framework \cite{fischler1981random, chum2008optimal} or robust global optimisation \cite{li2009consensus, enqvist2012robust, ask2013optimal, svarm2014accurate, enqvist2015tractable, svarm2016city} can be used to find the inlier set. Alternatively, outlier removal schemes can make the problem more tractable \cite{sim2006removing, olsson2010outlier, yu2011adversarial, chin2016guaranteed}. Other methods develop sophisticated matching strategies to avoid outlier correspondences at the outset \cite{li2010location, sattler2011fast, sattler2012improving, li2012worldwide}. However, these methods require some correct correspondences. For this reason, they are often only practical for 3D models that have been constructed using stereopsis or Structure-from-Motion (SfM). These models associate an image feature with each 3D point, facilitating inter-modality feature matching. Generic point-sets do not have this property; a point may lie anywhere on the underlying surfaces in a laser scan, not just where strong image gradients occur.

When correspondences are unknown, the problem becomes more challenging. For the 2D--2D case, problems such as correspondence-free rigid registration \cite{besl1992method,breuel2003implementation}, SfM \cite{dellaert2000structure,makadia2007correspondence,lin2012simultaneous} and relative camera pose \cite{fredriksson2016optimal} have been addressed. For the 2D--3D case, solution have been proposed for registering a collection of images \cite{paudel2015robust} or multiple cameras \cite{paudel2015lmi} to a 3D point-set. The more general problem, however, is pose estimation from a single image. David \etal \cite{david2004softposit} proposed the SoftPOSIT algorithm, which alternates correspondence assignment with an iterative pose update algorithm. Moreno-Noguer \etal \cite{moreno2008pose} proposed the BlindPnP algorithm, which represents the pose prior as a Gaussian mixture model from which a Kalman filter is initialised for matching. It outperformed SoftPOSIT when large amounts of clutter, occlusions and repetitive patterns were present. However, both are susceptible to local optima, require a pose prior and cannot guarantee global optimality.

Grimson \cite{grimson1990object} applied a RANSAC-like approach to the correspondence-free case, removing the need for a pose prior, but the method is not optimal and quickly becomes intractable as the number of points increase. In contrast, globally-optimal methods find a camera pose that is guaranteed to be an optimiser of an error function without requiring a pose prior, but tractability remains a challenge. A Branch-and-Bound (BB) \cite{land1960automatic} strategy may be applied in these cases, for which bounds need to be derived. For example, Breuel \cite{breuel2003implementation} used BB for 2D--2D registration problems, Hartley and Kahl \cite{hartley2009global} for optimal relative pose estimation by bounding the group of 3D rotations, Li and Hartley~\cite{li20073d} for rotation-only 3D--3D registration, Olsson \etal~\cite{olsson2009branch} for 3D--3D registration with known correspondences, Yang \etal~\cite{yang2016goicp} for full 3D--3D registration and Campbell and Petersson~\cite{campbell2016gogma} for robust 3D--3D registration. While not optimal, Jurie \cite{jurie1999solution} used an approach similar to BB for 2D--3D alignment with a linear approximation of perspective projection. Brown \etal \cite{brown2015globally} proposed a global and $\epsilon$-suboptimal method using BB. It finds a camera pose whose trimmed geometric error, the sum of angular distances between the bearings and their rotationally-closest 3D points, is within $\epsilon$ of the global minimum. While not susceptible to local minima, it requires the inlier fraction to be specified, which can rarely be known in advance, in order to trim outliers.

Our work is the first globally-optimal inlier set cardinality maximisation solution to the simultaneous pose and correspondence problem. It is guaranteed to find the exact global optimum without requiring a pose prior and is robust to 2D and 3D outliers while avoiding the distortion of trimming. The rest of the paper is organised as follows: we introduce the problem formulation in Section~\ref{sec:inlier_set_maximisation}, develop a parametrisation of the domain of 3D motions, a branching strategy and a derivation of the bounds in Section~\ref{sec:bb}, propose an algorithm for globally-optimal pose and correspondence in Section~\ref{sec:algorithm} and evaluate its performance in Section~\ref{sec:results}.

%%%%%%%%%%%%%%%%%%%%%%%%%%%%%%%%%%%%%%%%%%%%%%%%%%%%%%%%%%%%%%%%%%%%%%%%%%%%%%%%%%%%%%%%%%%%%%%
\section{Inlier Set Cardinality Maximisation}
\label{sec:inlier_set_maximisation}

Let $\bp \in \mathbb{R}^3$ be a 3D point and $\bbf \in \mathbb{R}^3$ be a bearing vector with unit norm, corresponding to a 2D point imaged by a calibrated camera. That is, $\bbf \propto \mathtt{K}^{-1}\hat{\bx}$ where $\mathtt{K}$ is the matrix of intrinsic camera parameters and $\hat{\bx}$ is the homogeneous image point. Given a set of points $\sP = \{\bp_j\}_{j=1}^{M}$ and bearing vectors $\sF = \{\bbf_i\}_{i=1}^{N}$ and an inlier threshold $\theta$, the objective is to find a rotation $\bR \in SO(3)$ and translation $\bt \in \mathbb{R}^3$ that maximises the cardinality $\nu$ of the inlier set $\sS_I$
\vspace{-2pt}
\begin{equation}
\label{eqn:inlier_maximisation}
\nu^* = \max_{\bR,\,\bt} |\sS_I|
\end{equation}
\vspace{-12pt}
\begin{equation}
\label{eqn:inlier_set_definition}
\sS_I = \{\bbf \in \sF \mid \exists \bp \in \sP : \angle(\bbf, \bR(\bp - \bt)) \leqslant \theta\}
\end{equation}
where $\angle(\cdot,\cdot)$ denotes the angular distance between vectors.
An equivalent formulation is given by
\vspace{-2pt}
\begin{equation}
\label{eqn:inlier_maximisation_function}
\nu^* = \max_{\bR,\,\bt} f(\bR, \bt)
\end{equation}
\vspace{-12pt}
\begin{equation}
\label{eqn:inlier_function}
f(\bR,\bt) = \sum_{\bbf \in \sF} \max_{\bp \in \sP} \indfn\big(\theta - \angle(\bbf, \bR(\bp - \bt))\big)
\end{equation}
where $\indfn(x) \defeq \indfn_{\mathbb{R}_{\ge 0}}(x)$ is the indicator function that has the value 1 for all elements of the non-negative real numbers and the value 0 otherwise. The optimal transformation parameters $\bR^{*}$ and $\bt^{*}$ allow us to find all correspondences $(\bbf_i, \bp_j)$ with respect to $\theta$ by identifying all pairs for which $\angle(\bbf_i, \bR^*(\bp_j - \bt^*)) \leqslant \theta$. We maximise the cardinality of the set of bearing vector inliers, not the set of 3D point inliers, to avoid the degenerate case of all points sharing the same bearing vector inlier, which occurs when the camera is translated far away from the point-set.

%%%%%%%%%%%%%%%%%%%%%%%%%%%%%%%%%%%%%%%%%%%%%%%%%%%%%%%%%%%%%%%%%%%%%%%%%%%%%%%%%%%%%%%%%%%%%%%
\section{Branch-and-Bound}
\label{sec:bb}

To solve the highly non-convex cardinality maximisation problem (\ref{eqn:inlier_maximisation}), the global optimisation technique of Branch-and-Bound (BB) \cite{land1960automatic} may be applied. To do so, a suitable means of parametrising and branching (partitioning) the function domain must be found, as well as an efficient way to calculate upper and lower bounds of the function for each branch which converge as the size of the branch tends to zero. While the bounds need to be computationally efficient to calculate, the time and memory efficiency of the algorithm also depends on how tight the bounds are, since tighter bounds reduce the search space quicker by allowing suboptimal branches to be pruned.

%%%%%%%%%%%%%%%%%%%%%%%%%%%%%%%%%%%%%%%%%%%%%%%%%%%%%%%%%%%%%%%%%%%%%%%%%%%%%%%%%%%%%%%%%%%%%%%
\subsection{Parametrising and Branching the Domain}
\label{sec:bb_branching}

To find a globally-optimal solution, the cardinality of the inlier set $\sS_I$ must be maximised over the domain of 3D motions, that is, the group $SE(3) = SO(3) \times \mathbb{R}^3$. However, the space of these transformations is unbounded, therefore we restrict the space of translations to be within the bounded set $\Omega_t$ in order to use BB. For a suitably large $\Omega_t$, it is reasonable to assume that the camera centre lies within $\Omega_t$. That is, we can assume that the camera is a finite distance from the 3D points. The domains are shown in Figure~\ref{fig:domain}.

\begin{figure}[!t]
\centering
\begin{subfigure}[]{0.495\columnwidth}
\centering
\def\svgwidth{0.65\columnwidth}
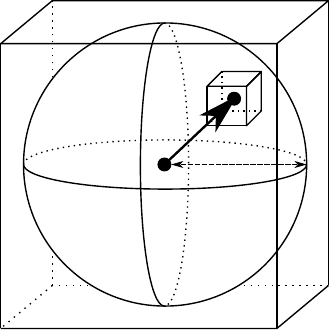
\caption{Rotation Domain $\Omega_r$}
\label{fig:domain_rotation}
\end{subfigure}
\hfill
\begin{subfigure}[]{0.495\columnwidth}
\centering
\def\svgwidth{0.65\columnwidth}
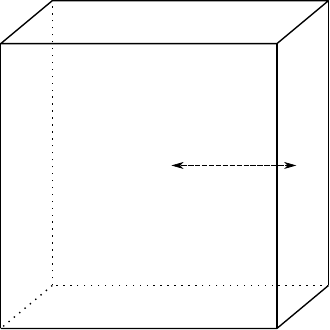
\caption{Translation Domain $\Omega_t$}
\label{fig:domain_translation}
\end{subfigure}
\caption{Parametrisation of $SE(3)$. (\subref{fig:domain_rotation})~The rotation space $SO(3)$ is parametrised by angle-axis 3-vectors in a solid radius-$\pi$ ball. (\subref{fig:domain_translation})~The translation space $\mathbb{R}^3$ is parametrised by 3-vectors bounded by a cuboid with half-widths $[\tau_x, \tau_y, \tau_z]$. The domain is branched into sub-cuboids as shown using nested octree data structures.}
\label{fig:domain}
\vspace{-10pt}
\end{figure}

Rotation space $SO(3)$ is minimally parametrised with angle-axis 3-vectors $\br$ with rotation angle $\|\br\|$ and rotation axis $\br/\|\br\|$. The notation $\bR_{\br} \in SO(3)$ is used to denote the rotation matrix obtained from the matrix exponential map of the skew-symmetric matrix $[\br]_{\times}$ induced by $\br$. The Rodrigues' rotation formula can be used to efficiently calculate this mapping. Using this parametrisation, the space of all 3D rotations can be represented as a solid ball of radius $\pi$ in $\mathbb{R}^3$. The mapping is one-to-one on the interior of the $\pi$-ball and two-to-one on the surface. For ease of manipulation, we use the 3D cube circumscribing the $\pi$-ball as the rotation domain $\Omega_r$~\cite{li20073d}. Translation space $\mathbb{R}^3$ is parametrised with 3-vectors in a bounded domain chosen as the cuboid $\Omega_t$ containing the bounding box of $\sP$. If the camera is known to be inside the 3D scene, $\Omega_t$ can be set to the bounding box, otherwise it is set to an expansion of the bounding box.

During BB, the domain is branched into sub-cuboids using nested octree data structures. They are defined as
\begin{equation}
\label{eqn:subcube_definition}
\sC(\mathbf{c},\boldsymbol{\delta}) \!=\! \{\bx \in \mathbb{R}^3 \mid \be_i\transpose(\bx - \mathbf{c}) \in [-\delta_i, \delta_i], i = 1,2,3\}
\end{equation}
where $\be_i$ is the $i$\textsuperscript{th} standard basis vector. To simplify the notation, we use $\sC_r = \sC(\br_0,\boldsymbol{\delta}_r)$ and $\sC_t = \sC(\bt_0,\boldsymbol{\delta}_t)$.

The uncertainty angle induced by a rotation and translation sub-cuboid on a point $\bp$ is shown in Figure~\ref{fig:uncertainty_angle}. The transformed point may lie anywhere within an uncertainty cone, with aperture angle equal to the sum of the rotation and translation uncertainty angles.

\begin{figure}[!t]
\centering
\begin{subfigure}[]{0.495\columnwidth}
\def\svgwidth{\columnwidth}
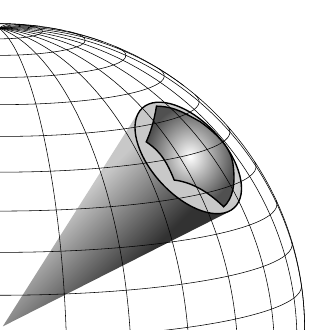
\caption{Rotation Uncertainty Angle}
\label{fig:uncertainty_angle_rotation}
\end{subfigure}
\hfill
\begin{subfigure}[]{0.495\columnwidth}
\def\svgwidth{\columnwidth}
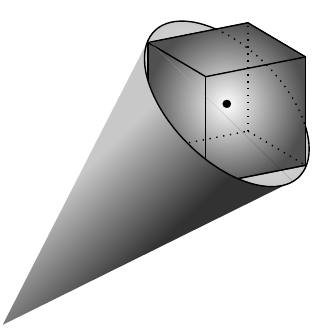
\caption{Translation Uncertainty Angle}
\label{fig:uncertainty_angle_translation}
\end{subfigure}
\caption{Uncertainty angles induced by rotation and translation sub-cubes. (\subref{fig:uncertainty_angle_rotation})~Rotation uncertainty angle $\psi_r$ for $\sC_r$. The optimal rotation of $\bp$ may be anywhere within the umbrella-shaped region, which is entirely contained by the cone defined by $\bR_{\br_{0}}\bp$ and $\psi_r$. (\subref{fig:uncertainty_angle_translation})~Translation uncertainty angle $\psi_t$ for $\sC_t$. The optimal translation of $\bp$ may be anywhere within the cuboidal region, which is entirely contained by the cone defined by $\bp - \bt_0$ and $\psi_t$.}
\label{fig:uncertainty_angle}
\vspace{-6pt}
\end{figure}

%%%%%%%%%%%%%%%%%%%%%%%%%%%%%%%%%%%%%%%%%%%%%%%%%%%%%%%%%%%%%%%%%%%%%%%%%%%%%%%%%%%%%%%%%%%%%%%
\subsection{Bounding the Branches}
\label{sec:bb_bounding}

The success of a BB algorithm is predicated on the quality of its bounds. For inlier set maximisation, the objective function (\ref{eqn:inlier_function}) needs to be bounded within a transformation domain. Some preparatory material is now presented.

%%%%%%%%%%%%%%%%%%%%%%%%%%%%%%%%%%%%%%%%%%%%%%%%%%%%%%%%%%%%%%%%%%%%%%%%%%%%%%%%%%%%%%%%%%%%%%%
% Rotation bounds

To bound the uncertainty angle due to rotation, Lemmas 1 and 2 from \cite{hartley2009global} are used. For reference, the relevant parts are merged into Lemma~\ref{lm:rotation_inequality}, as in \cite{yang2016goicp}. The lemma indicates that the angle between two rotated vectors is less than or equal to the Euclidean distance between their rotations' angle-axis representations in $\mathbb{R}^3$.
\begin{lemma}
\label{lm:rotation_inequality}
For an arbitrary vector $\bp$ and two rotations, represented as $\bR_{\br_{1}}$ and $\bR_{\br_{2}}$ in matrix form and $\br_{1}$ and $\br_{2}$ in angle-axis form,
\begin{equation}
\label{eqn:rotation_inequality}
\angle(\bR_{\br_{1}}\bp, \, \bR_{\br_{2}} \bp) \leqslant \|\br_{1} - \br_{2}\|.
\end{equation}
\end{lemma}

From this, the maximum angle between a vector $\bp$ rotated by $\br_0$ and $\bp$ rotated by $\br \in \sC_r$ can be found as follows.
\begin{lemma}
\label{lm:rotation_uncertainty_angle_weak}
(Weak rotation uncertainty angle) Given a 3D point $\bp$ and a rotation cube $\sC_r$ of half side-length $\delta_r$ centred at $\br_0$, then $\forall \br \in \sC_r$,
\begin{equation}
\label{eqn:rotation_uncertainty_angle_weak}
\angle(\bR_\br \bp, \bR_{\br_0} \bp) \leqslant \min(\sqrt{3}\delta_r,\pi) \defeq \psi_{r}^{w}(\sC_r).
\end{equation}	
\end{lemma}
\begin{proof}
Inequality (\ref{eqn:rotation_uncertainty_angle_weak}) can be derived as follows:
\begin{align}
\angle(\bR_\br \bp, \bR_{\br_0} \bp)
&\leqslant \min({\|\br-\br_0\|},{\pi})
\label{eqn:rotation_uncertainty_angle_weak_1}\\
&\leqslant \min({\sqrt{3}\delta_r},{\pi})
\label{eqn:rotation_uncertainty_angle_weak_2}
\end{align}
where (\ref{eqn:rotation_uncertainty_angle_weak_1}) follows from Lemma~\ref{lm:rotation_inequality} and the maximum possible angle and (\ref{eqn:rotation_uncertainty_angle_weak_2}) follows from $\max \|\br - \br_0\| = \sqrt{3} \delta_r$ (the half space diagonal of the rotation cube) for $\br \in \sC_r$.
\end{proof}

However, a tighter bound can be found by observing that a point rotated about an axis parallel to the point is not displaced. To exploit this, we maximise the angle $\angle(\bR_\br \bp, \bR_{\br_0} \bp)$ over the surface $\sS_r$ of the cube $\sC_r$.
\begin{lemma}
\label{lm:rotation_uncertainty_angle}
(Rotation uncertainty angle) Given a 3D point $\bp$ and a rotation cube $\mathcal{C}_r$ centred at $\br_0$ with surface $\sS_r$, then $\forall \br \in \mathcal{C}_r$,
\begin{equation}
\label{eqn:rotation_uncertainty_angle}
\angle(\bR_\br \bp, \bR_{\br_0} \bp) \leqslant \min(\max_{\br \in \sS_{r}}\angle(\bR_{\br} \bp, \bR_{\br_0} \bp),\pi) \defeq \psi_{r}(\bp, \mathcal{C}_r).
\end{equation}
\end{lemma}
\begin{proof}
Inequality (\ref{eqn:rotation_uncertainty_angle}) can be derived as follows:
\begin{align}
\angle(\bR_\br \bp, \bR_{\br_0} \bp)
&\leqslant \min(\max_{\br \in \sC_{r}}\angle(\bR_{\br} \bp, \bR_{\br_0} \bp),\pi)
\label{eqn:rotation_uncertainty_angle_1}\\
&= \min(\max_{\br \in \sS_{r}}\angle(\bR_{\br} \bp, \bR_{\br_0} \bp),\pi)
\label{eqn:rotation_uncertainty_angle_2}
\end{align}
where (\ref{eqn:rotation_uncertainty_angle_2}) is a consequence of the order-preserving mapping, with respect to the radial angle, from the convex cube of angle-axis vectors to the spherical surface patch (see Figure \ref{fig:uncertainty_angle_rotation}), since the mapping is obtained by projecting from the centre of the sphere to the surface of the sphere. See the appendix for further details.
\end{proof}

%%%%%%%%%%%%%%%%%%%%%%%%%%%%%%%%%%%%%%%%%%%%%%%%%%%%%%%%%%%%%%%%%%%%%%%%%%%%%%%%%%%%%%%%%%%%%%%
% Translation bounds

The uncertainty angle due to translation can be bounded by observing that the translated points form a cube (Figure~\ref{fig:uncertainty_angle_translation}). When the cube does not contain the origin, the angle can be found by maximising over the cube vertices.
\begin{lemma}
\label{lm:translation_uncertainty_angle}
(Translation uncertainty angle) Given a 3D point $\bp$ and a translation cube $\sC_t$ centred at $\bt_0$ with vertices $\sV_t$, then $\forall \bt \in \sC_t$,
\begin{align}
\label{eqn:translation_uncertainty_angle}
\angle(\bp - \bt, \bp - \bt_0) &\leqslant
\begin{dcases}
\max_{\bt \in \sV_t}\angle(\bp - \bt, \bp - \bt_0) & \text{if } \bp \notin \sC_t\\
\pi & \text{else}
\end{dcases}\nonumber\\
&\defeq \psi_{t}(\bp, \sC_t).
\end{align}
\end{lemma}
\begin{proof}
Observe that for $\bp \in \sC_t$, the cube containing all translated points $\bp - \bt$ also contains the origin. Therefore $\bp - \bt$ can be proportional to $-(\bp - \bt_0)$ and thus the maximum angle is $\pi$. For $\bp \notin \sC_t$,
\begin{align}
\angle(\bp - \bt, \bp - \bt_0) &\leqslant \max_{\bt \in \sC_t}\angle(\bp - \bt, \bp - \bt_0)
\label{eqn:translation_uncertainty_angle_1}\\
&= \max_{\bt \in \sV_t}\angle(\bp - \bt, \bp - \bt_0)
\label{eqn:translation_uncertainty_angle_2}
\end{align}
where (\ref{eqn:translation_uncertainty_angle_2}) follows from the convexity of the angle function in this domain. The maximum of a convex function over a convex set must occur at one of its extreme points (the vertices). Geometrically, the cube $\bp - \bt$ projects to a spherical hexagon on the unit sphere. The maximum geodesic from a point in the hexagon to any other is to a vertex.
\end{proof}

To avoid the non-physical case where a 3D point is located within a very small value $\zeta$ of the camera centre we restrict the translation domain such that $\Omega_t^{\prime} = \Omega_t \cap \{\bt \in \mathbb{R}^3 \mid \|\bp - \bt\| \geqslant \zeta, \forall \bp \in \sP\}$.

The translation bound from \cite{brown2015globally} encloses a translation cube with a sphere of radius $\rho_t = \sqrt{3} \delta_t$ and is given by
\begin{equation}
\label{eqn:translation_uncertainty_angle_weak}
\psi_{t}^{w}(\bp, \sC_t) \defeq
\begin{cases}
\arcsin\left(\frac{\rho_t}{\| \bp - \bt_0 \|}\right) & \text{if } \rho_t \leqslant \| \bp - \bt_0 \|\\
\pi & \text{else.}
\end{cases}
\end{equation}
Our bound is tighter with a maximum difference of $117^{\circ}$ for cubes and greater for cuboids. Figure~\ref{fig:translation_bound_comparisons} compares both translation bounds across a range of values.

\begin{figure}[!t]
\centering
\begin{subfigure}[]{1.74in}
\includegraphics[trim=0 0 0 0, clip, height=1.36in]{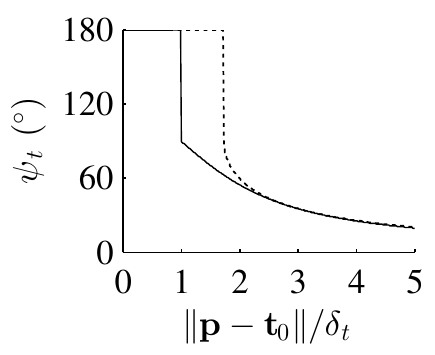}
\caption{Ray through face centre}
\label{fig:translation_bound_comparisons_1}
\end{subfigure}
\hfill
\begin{subfigure}[]{1.5in}
\includegraphics[trim=0 0 0 0, clip, height=1.36in]{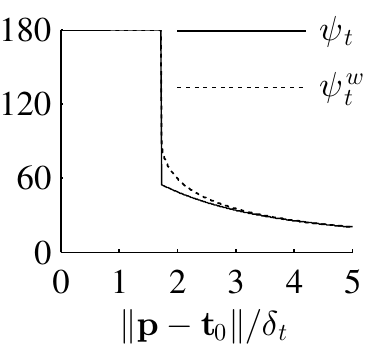}
\caption{Ray through vertex}
\label{fig:translation_bound_comparisons_2}
\end{subfigure}
\caption{Comparison of translation bounds when the cube centre lies along a ray from the origin towards (\subref{fig:translation_bound_comparisons_1}) any face centre and (\subref{fig:translation_bound_comparisons_2}) any vertex. Our bound $\psi_t$ is tighter across the entire domain.}
\label{fig:translation_bound_comparisons}
\end{figure}

%%%%%%%%%%%%%%%%%%%%%%%%%%%%%%%%%%%%%%%%%%%%%%%%%%%%%%%%%%%%%%%%%%%%%%%%%%%%%%%%%%%%%%%%%%%%%%%
% Bounds

The preceding lemmas are used to bound the objective function (\ref{eqn:inlier_function}) within a transformation domain $\sC_r\times \sC_t$. For brevity, we use the notation $\bp_{\bt}^{\br} \defeq \bR_{\br}(\bp - \bt)$, $\bp_{\bt} \defeq \bp - \bt$ and $\bbf_{\br} \defeq (\bR_{\br})^{-1}\bbf$.
\begin{theorem}
\label{thm:lower_bound}
(Lower bound) For the domain $\sC_r \times \sC_t$ centred at $(\br_0, \bt_0)$, the lower bound of the inlier set cardinality can be chosen as
\begin{equation}
\label{eqn:lower_bound}
\ubar{f} (\bR_{\br}, \bt) \defeq f (\bR_{\br_0}, \bt_0).
\end{equation}
\end{theorem}
\begin{proof}
The validity of the lower bound follows from
\begin{equation}
\max_{\br,\,\bt} f(\bR_{\br}, \bt) \geqslant f (\bR_{\br_0}, \bt_0).
\label{eqn:lower_bound_1}
\end{equation}
That is, the function value at a specific point within the domain is less than or equal to the maximum.
\end{proof}

\begin{theorem}
\label{thm:upper_bound}
(Upper bound) For the domain $\sC_r \times \sC_t$ centred at $(\br_0, \bt_0)$, the upper bound of the inlier set cardinality can be chosen as
\begin{align}
\label{eqn:upper_bound}
\bar{f} (\bR_{\br}, \bt) \!&\defeq\!\! \sum_{\bbf \in \sF} \max_{\bp \in \sP} \indfn\big(\theta \!-\!\! \angle(\bbf, \bp_{\bt_0}^{\br_0}) \!+\! \psi_r(\bbf, \sC_r) \!+\! \psi_t(\bp, \sC_t)\big).
\end{align}
\end{theorem}
\begin{proof}
Observe that $\forall(\br,\bt) \in (\sC_r \times \sC_t)$,
\begin{align}
\angle(\bbf, \bp_{\bt}^{\br}) &\geqslant \angle(\bbf, \bp_{\bt_0}^{\br_0}) - \angle(\bbf_{\br}, \bbf_{\br_{0}}) - \angle(\bp_{\bt}, \bp_{\bt_0})
\label{eqn:upper_bound_1}\\
&\geqslant \angle(\bbf, \bp_{\bt_0}^{\br_0}) - \psi_r(\bbf, \sC_r) - \psi_t(\bp, \sC_t)
\label{eqn:upper_bound_2}
\end{align}
where (\ref{eqn:upper_bound_1}) follows from the triangle inequality in spherical geometry (see Figure \ref{fig:spherical_triangle}) and (\ref{eqn:upper_bound_2}) follows from Lemmas \ref{lm:rotation_uncertainty_angle} and \ref{lm:translation_uncertainty_angle}. Substituting (\ref{eqn:upper_bound_2}) into (\ref{eqn:inlier_function}) completes the proof.
\end{proof}

\begin{figure}[!t]
\centering
\def\svgwidth{0.6\columnwidth}
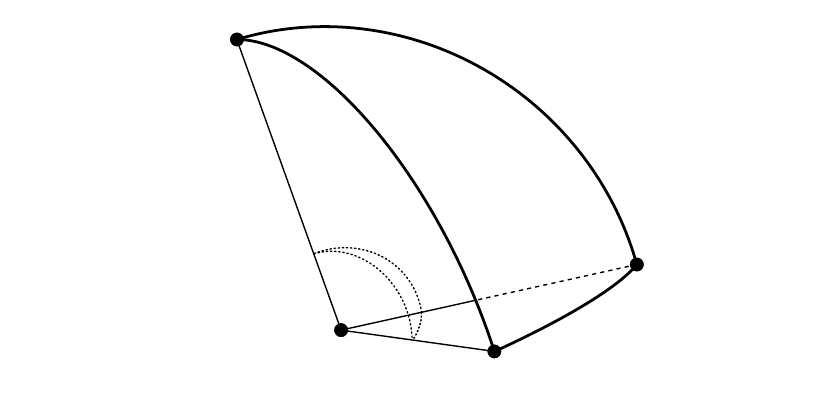
\caption{The triangle inequality in spherical geometry, given by $\beta \leqslant \alpha + \gamma$ or $\angle(\bbf_{\br}, \bp_{\bt_0}) \leqslant \angle(\bbf_{\br}, \bp_{\bt}) + \angle(\bp_{\bt}, \bp_{\bt_0})$. The transformed points have been normalised to lie on the unit sphere.}
\label{fig:spherical_triangle}
\end{figure}

By inspecting the translation component of Theorem~\ref{thm:upper_bound}, a tighter upper bound may be found by removing one of the two applications of the triangle inequality. A similar approach cannot be taken for the rotation component since $\bR_{\br}\bp$ is a complex surface due to the nonlinear conversion from angle-axis to rotation matrix representations. To reduce computation, it is only necessary to evaluate this tighter bound when $\angle(\bbf, \bp_{\bt_0}^{\br_0}) \leqslant \theta + \psi_r(\bbf, \sC_r) + \psi_t(\bp, \sC_t)$, since otherwise the point is definitely an outlier and does not need to be investigated further.
\begin{theorem}
\label{thm:upper_bound_tighter}
(Tighter upper bound) For the domain $\sC_r \times \sC_t$ centred at $(\br_0, \bt_0)$, the upper bound of the inlier set cardinality can be chosen as
\begin{equation}
\label{eqn:upper_bound_tighter}
\bar{f} (\bR_{\br}, \bt) \defeq \sum_{\bbf \in \sF} \max_{\bp \in \sP} \Gamma(\bbf,\bp)
\end{equation}
\begin{equation}
\label{eqn:upper_bound_tighter_gamma}
\Gamma(\bbf, \bp) = \max_{\bt \in \sC_t}\indfn\big(\theta - \angle(\bbf, \bp_{\bt}^{\br_0}) + \psi_r(\bbf, \sC_r)\big).
\end{equation}
\end{theorem}
\begin{proof}
Observe that $\forall(\br,\bt) \in (\sC_r \times \sC_t)$,
\begin{align}
\!\!\!\!\!\indfn\big(\theta \!-\! \angle(\bbf, \bp_{\bt}^{\br})\big) &\leqslant \indfn\big(\theta - \angle(\bbf, \bp_{\bt}^{\br_0}) + \angle(\bbf_{\br}, \bbf_{\br_0})\big)
\label{eqn:upper_bound_tighter_1}\\
&\leqslant
\max_{\bt \in \sC_t}\indfn\big(\theta - \angle(\bbf, \bp_{\bt}^{\br_0}) + \psi_r(\bbf, \sC_r)\big)
\label{eqn:upper_bound_tighter_2}
\end{align}
where (\ref{eqn:upper_bound_tighter_1}) follows from the triangle inequality in spherical geometry and (\ref{eqn:upper_bound_tighter_2}) follows from Lemma \ref{lm:rotation_uncertainty_angle} and maximising over $\bt$. Substituting (\ref{eqn:upper_bound_tighter_2}) into (\ref{eqn:inlier_function}) completes the proof.
\end{proof}

$\Gamma$ may be evaluated by observing that the minimum angle between a ray $\bbf$ and a cube $\bp_{\bt}^{\br_0}$ is zero if the ray passes through the cube and is otherwise the angle between the ray and the point on the skeleton of the cube (vertices and edges) with least angular displacement from $\bbf$. Thus, for the translation domain $\sC_t$ with skeleton $\sSk_t$,
\begin{gather}
\label{eqn:upper_bound_tighter_tmax}
\Gamma =
\begin{dcases}
\max_{\bt \in \sSk_t}\indfn\big(\theta - \angle(\bbf, \bp_{\bt}^{\br_0}) + \psi_r\big) & \text{if } \angle(\bbf, \bp_{\bt_0}^{\br_0}) > \psi_{t}\\
1 & \text{else.}
\end{dcases}\raisetag{14pt}
\end{gather}

%%%%%%%%%%%%%%%%%%%%%%%%%%%%%%%%%%%%%%%%%%%%%%%%%%%%%%%%%%%%%%%%%%%%%%%%%%%%%%%%%%%%%%%%%%%%%%%
\section{The GOPAC Algorithm}
\label{sec:algorithm}

The Globally-Optimal Pose And Correspondences (GOPAC) algorithm for a calibrated camera is outlined in Algorithms \ref{alg:gopnp} and \ref{alg:rotation_search}. As in \cite{yang2016goicp}, we employ a nested branch-and-bound structure for computational efficiency. In the outer breadth-first BB search, upper and lower bounds are found for each translation cuboid $\sC_{t} \in \Omega_t$ by running an inner BB search over rotation space $SO(3)$ (denoted RBB). The upper bound $\bar{\nu} \defeq \bar{\nu}_{t}$ (\ref{eqn:upper_bound}) of $\sC_{t}$ is found by running RBB until convergence with the following bounds
\begin{align}
\ubar{\nu}_{r} &\defeq \sum_{\bbf \in \sF} \max_{\bp \in \sP} \indfn\big(\theta - \angle(\bbf, \bp_{\bt_0}^{\br_0}) + \psi_t(\bp)\big)\label{eqn:lower_bound_nested}\\
\bar{\nu}_{r} &\defeq \sum_{\bbf \in \sF} \max_{\bp \in \sP} \indfn\big(\theta - \angle(\bbf, \bp_{\bt_0}^{\br_0}) + \psi_t(\bp) + \psi_r(\bbf)\big).\label{eqn:upper_bound_nested}
\end{align}
The tighter upper bound (\ref{eqn:upper_bound_tighter}) instead uses
\begin{align}
\ubar{\nu}_{r} &\defeq \sum_{\bbf \in \sF} \max_{\bp \in \sP, \bt \in \sC_t} \indfn\big(\theta - \angle(\bbf, \bp_{\bt}^{\br_0})\big)\label{eqn:lower_bound_nested_tighter}\\
\bar{\nu}_{r} &\defeq \sum_{\bbf \in \sF} \max_{\bp \in \sP, \bt \in \sC_t} \indfn\big(\theta - \angle(\bbf, \bp_{\bt}^{\br_0}) + \psi_r(\bbf)\big).\label{eqn:upper_bound_nested_tighter}
\end{align}
The lower bound $\ubar{\nu} \defeq \ubar{\nu}_{t}$ (\ref{eqn:lower_bound}) is found by running RBB using bounds (\ref{eqn:lower_bound_nested}) and (\ref{eqn:upper_bound_nested}) with $\psi_t$ set to zero.

The nested structure has better memory and computational efficiency than directly branching over 6D transformation space, since it maintains a queue for each 3D sub-problem, rather than one for the entire 6D problem. This requires significantly fewer simultaneously enqueued sub-cubes. Moreover, with rotation search nested inside translation search, $\psi_t$ only has to be calculated once per translation~$\bt$, not once per pose $(\br, \bt)$, and $\sF$ can be rotated (by $\bR^{-1}$) instead of $\sP$ which typically has more elements. This makes it possible to precompute the rotated bearing vectors and rotation bounds for the top five levels of the rotation octree to reduce the amount of computation required in the inner BB subroutine.

\begin{algorithm}[!t]
\begin{algorithmic}[1]
\Require bearing vector set $\sF$, point set $\sP$, inlier threshold $\theta$, initial domains $\Omega_r$ and $\Omega_t$

\Ensure optimal number of inliers $\nu^*$, camera pose $(\br^*,\bt^*)$ and 2D--3D correspondences

\State $\nu^* \gets 0$\label{alg:nu_initialisation}
\State Add translation domain $\Omega_t$ to priority queue $Q_t$
\Loop
\State Update greatest upper bound $\bar{\nu}_t$ from $Q_t$\label{alg:global_upper_bound}
\State Get cuboid $\sC_t$ with greatest width $\delta_{tx}$ from $Q_t$\label{alg:priority_queue}
\State \textbf{if} $\nu^* \geqslant \bar{\nu}_t$ \textbf{then} terminate \label{alg:stopping_criterion}
\ForAll{sub-cuboids $\sC_{ti} \in \sC_t$}\label{alg:branching}
\State $(\ubar{\nu}_{ti}, \br) \gets \text{RBB}(\nu^*, \bt_{0i}, \psi_t = 0)$\label{alg:lower_bound}
\State \textbf{if} $\nu^* < 2\ubar{\nu}_{ti}$ \textbf{then} $(\nu^*, \br^*, \bt^*) \gets \text{P}n\text{P}(\br, \bt_{0i})$\label{alg:pnp}\label{alg:update_criterion}
\State $\bar{\nu}_{ti} \gets \text{RBB}(\nu^*, \bt_{0i}, \psi_t)$\label{alg:upper_bound}
\State \textbf{if} $\nu^* < \bar{\nu}_{ti}$ \textbf{then} add $\sC_{ti}$ to queue $Q_t$\label{alg:push_queue}
\EndFor
\EndLoop
\end{algorithmic}
\caption{GOPAC: a branch-and-bound algorithm for globally-optimal camera pose \& correspondence estimation}
\label{alg:gopnp}
\end{algorithm}

\begin{algorithm}[!t]
\begin{algorithmic}[1]
\Require bearing vector set $\sF$, point set $\sP$, inlier threshold $\theta$, initial domain $\Omega_r$, best-so-far cardinality $\nu^*$, translation $\bt_0$, translation uncertainty $\psi_t$

\Ensure optimal number of inliers $\nu_r^*$ and rotation $\bR^*$

\State $\nu_r^* \gets \nu^*$\label{alg:nu_initialisation_2}
\State Add rotation domain $\Omega_r$ to priority queue $Q_r$
\Loop
\State Read cube $\sC_r$ with greatest upper bound $\bar{\nu}_r$ from $Q_r$\label{alg:priority_queue_2}
\State \textbf{if} $\nu_r^* \geqslant \bar{\nu}_r$ \textbf{then} terminate \label{alg:stopping_criterion_2}
\ForAll{sub-cubes $\sC_{ri} \in \sC_r$}\label{alg:branching_2}
\State Calculate $\ubar{\nu}_{ri}$ by (\ref{eqn:lower_bound_nested}) or (\ref{eqn:lower_bound_nested_tighter})\label{alg:lower_bound_2}
\State \textbf{if} $\nu_r^* < \ubar{\nu}_{ri}$ \textbf{then} $\nu_r^* \gets \ubar{\nu}_{ri}, \br^* \gets \br_0$\label{alg:update_criterion_2}
\State Calculate $\bar{\nu}_{ri}$ by (\ref{eqn:upper_bound_nested}) or (\ref{eqn:upper_bound_nested_tighter})\label{alg:upper_bound_2}
\State \textbf{if} $\nu_r^* < \bar{\nu}_{ri}$ \textbf{then} add $\sC_{ri}$ to queue $Q_r$\label{alg:push_queue_2}
\EndFor
\EndLoop
\end{algorithmic}
\caption{RBB: a rotation search subroutine for GOPAC}
\label{alg:rotation_search}
\end{algorithm}

Line \ref{alg:pnp} of Algorithm \ref{alg:gopnp} shows how local optimisation is incorporated to refine the camera pose, in a similar manner to \cite{yang2016goicp,brown2015globally}. Whenever the BB algorithm finds a sub-cube pair $(\sC_r, \sC_t)$ with a greater lower bound $\ubar{\nu}$ than half the best-so-far cardinality $\nu^*$, the P$n$P problem is solved, with correspondences given by the inlier pairs at the pose $(\br_0, \bt_0)$. We use nonlinear optimisation \cite{kneip2014opengv}, minimising the sum of angular distances between corresponding bearing vectors and points, and update $\nu^*$ if a larger $\nu$ is found. In this way, BB and P$n$P collaborate, with P$n$P finding the best pose given correspondences and BB guiding the search for correspondences. P$n$P accelerates convergence since the faster $\nu^*$ is increased, the sooner sub-cubes (with $\bar{\nu} \leqslant \nu^*$) can be culled (Alg. \ref{alg:gopnp} Line~\ref{alg:push_queue}). SoftPOSIT \cite{david2004softposit} is also applied at this stage to jump to the nearest local maxima.

As just observed, a large $\nu^*$ reduces runtime. Therefore, if the user knows a lower bound on the number of 2D inliers, $\nu^*$ can be initialised to this value. However, this is rarely known. Instead, our algorithm implements an optional guess-and-verify approach, without loss of optimality or objective function distortion, which provides especial benefit when 2D outliers are rare: set $\nu^* = n$; run GOPAC; stop if an optimality guarantee is found, otherwise $n \gets \max(n - s, 0)$ and repeat. We initialise $n = N - 1$ and $s = \left\lceil{0.1N}\right\rceil$.

We also provide a multi-threaded implementation, where the initial translation domain is divided into sub-domains and GOPAC is run for each in separate CPU threads. The algorithm returns the largest $\nu^*$ and the associated pose and correspondences. While not supplied, a massively parallel implementation on a GPU is very feasible. Further algorithmic details are provided in the appendix.

%%%%%%%%%%%%%%%%%%%%%%%%%%%%%%%%%%%%%%%%%%%%%%%%%%%%%%%%%%%%%%%%%%%%%%%%%%%%%%%%%%%%%%%%%%%%%%%
\section{Results}
\label{sec:results}

% Competitors:
The GOPAC algorithm was evaluated with respect to the baseline RANSAC \cite{fischler1981random}, SoftPOSIT \cite{david2004softposit} and BlindPnP \cite{moreno2008pose} algorithms, denoted GP, RS, SP and BP respectively, with synthetic and real data. The RANSAC approach uses the OpenGV framework \cite{kneip2014opengv} and the P3P algorithm \cite{kneip2011novel} with randomly-sampled correspondences. Since SoftPOSIT and BlindPnP require pose priors to function, we use a torus prior in the synthetic experiments. In general, the space of camera poses is much larger than the restrictive torus prior and a good prior can rarely be known in advance. Except where otherwise specified, the inlier threshold $\theta$ was set to $1^\circ$, the rotation and translation bounds (\ref{eqn:rotation_uncertainty_angle}) and (\ref{eqn:translation_uncertainty_angle}) were used, SoftPOSIT and nonlinear P$n$P refinement were applied and multithreading was not used. It is crucial to observe that finding the global optimum does not necessarily imply finding the ground-truth transformation. There may be multiple global optima, particularly in the case of symmetries, and noise may create false optima.

\subsection{Synthetic Data Experiments}
\label{sec:results_synthetic}
To evaluate our algorithm in a setting where true priors can be applied, we performed 50 independent Monte Carlo simulations per parameter setting, using the framework of \cite{moreno2008pose}: $M$ random 3D points were generated from $[-1,1]^3$; a fraction $\omega_{\text{3D}}$ of the 3D points were randomly selected as outliers to model occlusion; the inliers were projected to a virtual image; normal noise was added with $\sigma=2$ pixels; and random points were added to the image such that a fraction $\omega_{\text{2D}}$ of the 2D points were outliers. To facilitate fair comparison with SoftPOSIT and BlindPnP, we use a pose prior for these experiments. The torus prior constrains the camera centre to a torus around the point-set with the optical axis directed towards the model, as in \cite{moreno2008pose}. BlindPnP represents the poses with a 20 component Gaussian mixture model, the means of which are used to initialise SoftPOSIT, as in \cite{moreno2008pose}. GOPAC is given a set of translation cubes which approximate the torus and is not given the rotation priors.

The results are shown in Figures \ref{fig:random_torus_3doutliers_m} and \ref{fig:random_torus_3d2doutliers}. We repeated the experiments for the repetitive CAD structure shown in Figure \ref{fig:models}, with results shown in Figure \ref{fig:cad_torus_3d2doutliers}. Two success rates are reported: the fraction of trials where the true maximum number of inliers was found and the fraction where the correct pose was found, where the angle between the output rotation and the ground truth rotation is less than $0.1$ radians and the camera centre error $\|\bt - \bt_{\text{GT}}\| / \| \bt_{\text{GT}} \|$ relative to the ground truth $\bt_{\text{GT}}$ is less than $0.1$, as in \cite{moreno2008pose}. The 2D and 3D outlier fractions were fixed to 0 when not being varied and multithreading was used in the 2D outlier experiments. GOPAC outperforms the other methods, reliably finding the global optimum while still being relatively efficient, particularly when the fraction of 2D outliers is low. For the repetitive CAD structure, while GOPAC finds the globally optimal number of inliers in all cases, the pose is occasionally incorrect when $75\%$ of the 3D points are occluded, due to the highly symmetric nature of the model.

\begin{figure}[!t]
\centering
\begin{subfigure}[]{0.25\columnwidth}
\includegraphics[trim=2pt 0pt 4pt 4pt, clip=true, scale=.55]{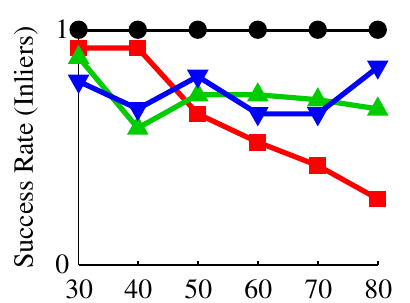}

\includegraphics[trim=2pt 0pt 4pt 4pt, clip=true, scale=.55]{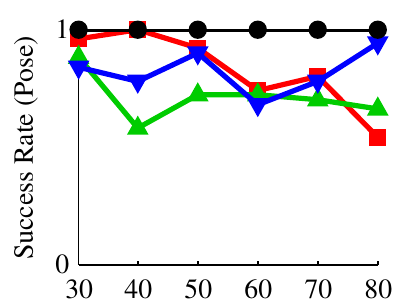}

\includegraphics[trim=2pt 0pt 4pt 4pt, clip=true, scale=.55]{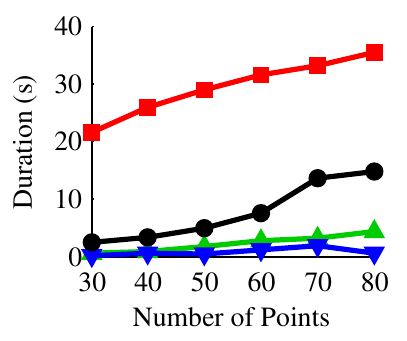}
\vspace{-16pt}
\caption{$\omega_{\text{3D}} = 0$}
\end{subfigure}
\hfill
\begin{subfigure}[]{0.24\columnwidth}
\includegraphics[trim=0pt 0pt 4pt 4pt, clip=true, scale=.55]{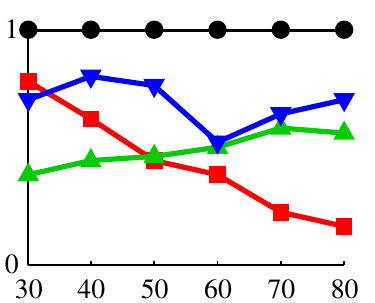}

\includegraphics[trim=0pt 0pt 4pt 4pt, clip=true, scale=.55]{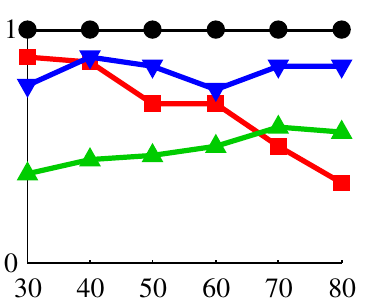}

\includegraphics[trim=0pt 0pt 4pt 4pt, clip=true, scale=.55]{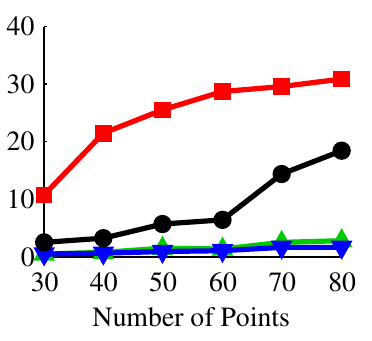}
\vspace{-16pt}
\caption{$\omega_{\text{3D}} = 0.25$}
\end{subfigure}
\hfill
\begin{subfigure}[]{0.24\columnwidth}
\includegraphics[trim=0pt 0pt 4pt 4pt, clip=true, scale=.55]{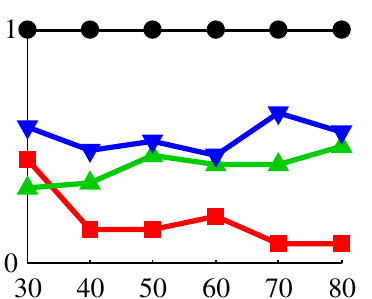}

\includegraphics[trim=0pt 0pt 4pt 4pt, clip=true, scale=.55]{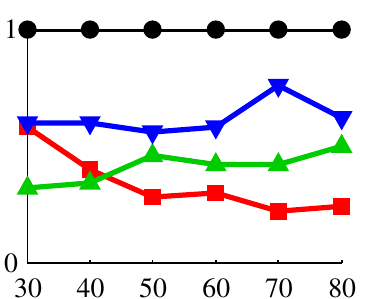}

\includegraphics[trim=0pt 0pt 4pt 4pt, clip=true, scale=.55]{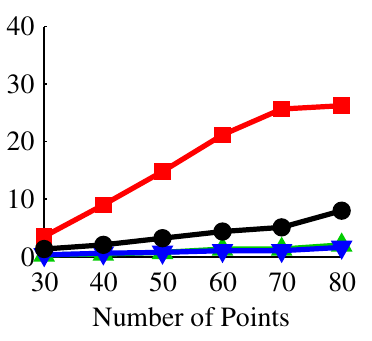}
\vspace{-16pt}
\caption{$\omega_{\text{3D}} = 0.5$}
\end{subfigure}
\hfill
\begin{subfigure}[]{0.24\columnwidth}
\includegraphics[trim=0pt 0pt 4pt 4pt, clip=true, scale=.55]{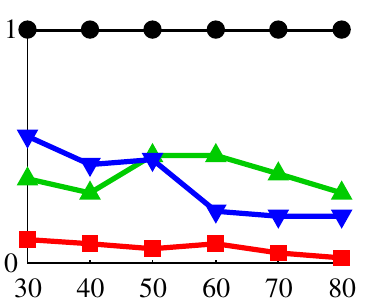}

\includegraphics[trim=0pt 0pt 4pt 4pt, clip=true, scale=.55]{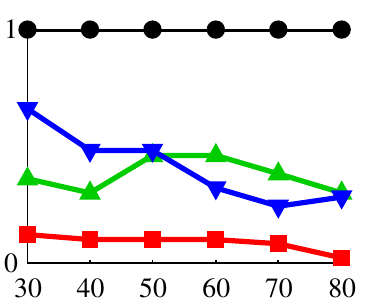}

\includegraphics[trim=0pt 0pt 4pt 4pt, clip=true, scale=.55]{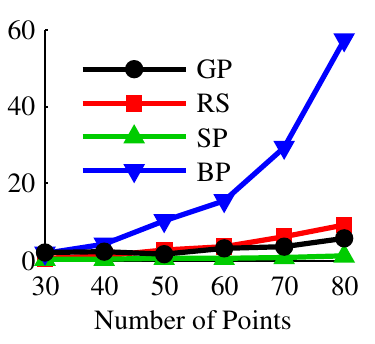}
\vspace{-16pt}
\caption{$\omega_{\text{3D}} = 0.75$}
\end{subfigure}
\caption{Mean success rates and median runtimes with respect to the number of random 3D points and the 3D outlier fraction, for 50 Monte Carlo simulations per parameter value with the torus prior.}
\label{fig:random_torus_3doutliers_m}
\end{figure}

\begin{figure}[!t]
\centering
\begin{subfigure}[]{0.505\columnwidth}
\includegraphics[trim=2pt 0pt 5pt 4pt, clip=true, scale=.55]{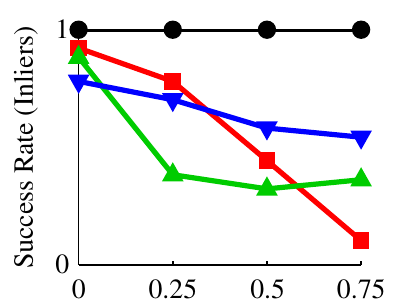}
\hfill
\includegraphics[trim=1pt 0pt 5pt 4pt, clip=true, scale=.55]{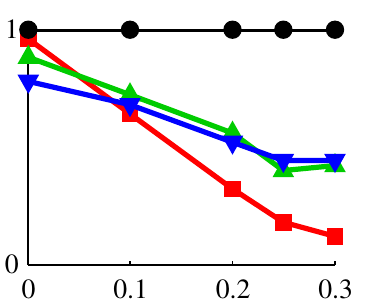}
\end{subfigure}
\hfill
\begin{subfigure}[]{0.485\columnwidth}
\includegraphics[trim=1pt 0pt 5pt 4pt, clip=true, scale=.55]{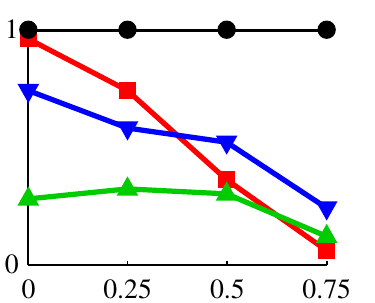}
\hfill
\includegraphics[trim=1pt 0pt 5pt 4pt, clip=true, scale=.55]{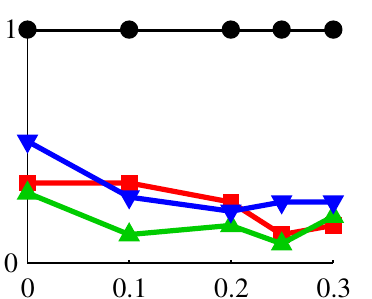}
\end{subfigure}

\begin{subfigure}[]{0.505\columnwidth}
\includegraphics[trim=2pt 0pt 5pt 4pt, clip=true, scale=.55]{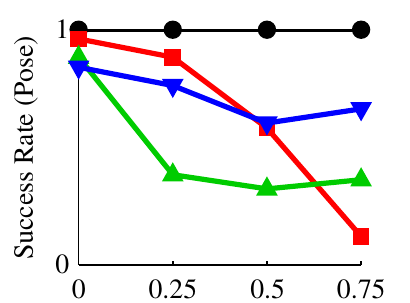}
\hfill
\includegraphics[trim=1pt 0pt 5pt 4pt, clip=true, scale=.55]{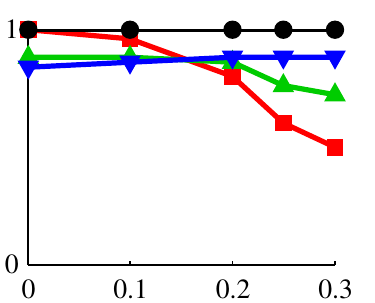}
\end{subfigure}
\hfill
\begin{subfigure}[]{0.485\columnwidth}
\includegraphics[trim=1pt 0pt 5pt 4pt, clip=true, scale=.55]{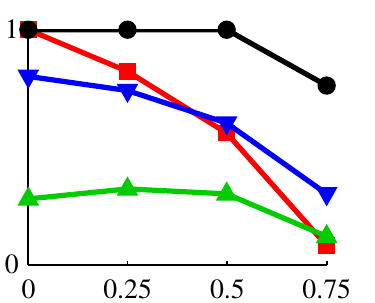}
\hfill
\includegraphics[trim=1pt 0pt 5pt 4pt, clip=true, scale=.55]{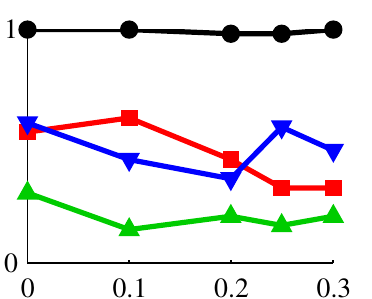}
\end{subfigure}

\begin{subfigure}[]{0.505\columnwidth}
\includegraphics[trim=2pt 0pt 5pt 4pt, clip=true, scale=.55]{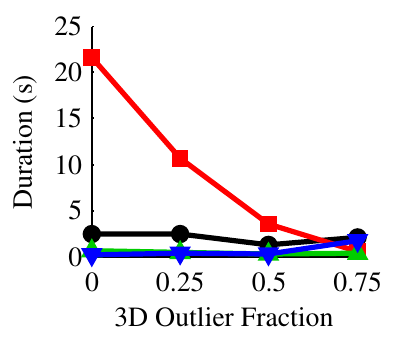}
\hfill
\includegraphics[trim=1pt 0pt 5pt 4pt, clip=true, scale=.55]{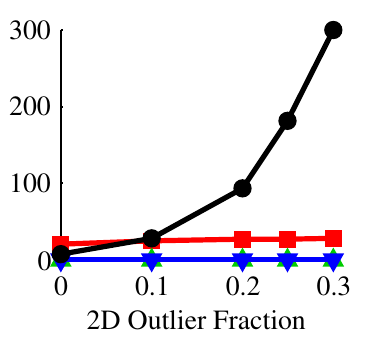}
\caption{Random Points $M = 30$}
\label{fig:random_torus_3d2doutliers}
\end{subfigure}
\hfill
\begin{subfigure}[]{0.485\columnwidth}
\includegraphics[trim=1pt 0pt 5pt 4pt, clip=true, scale=.55]{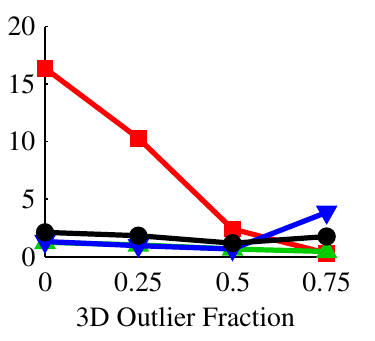}
\hfill
\includegraphics[trim=1pt 0pt 5pt 4pt, clip=true, scale=.55]{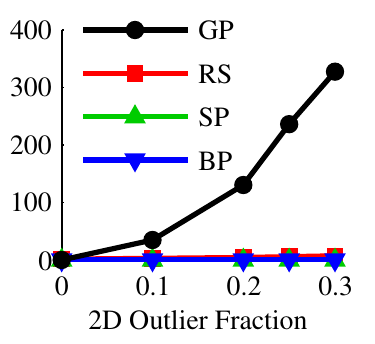}
\caption{CAD Structure $M = 27$}
\label{fig:cad_torus_3d2doutliers}
\end{subfigure}
\caption{Mean success rates and median runtimes with respect to the 3D and 2D outlier fractions for the random points and CAD structure datasets, for 50 Monte Carlo simulations per parameter value with the torus prior.}
\label{fig:randomcad_torus_3d2doutliers}
\vspace{-12pt}
\end{figure}

The evolution of the global lower and upper bounds is shown in Figure~\ref{fig:bound_evolution}: BB and P$n$P collaborate to increase the lower bound with BB guiding the search into better convergence basins and P$n$P refining the bound by jumping to the nearest local maximum (the staircase pattern). The majority of the time is spent decreasing the upper bound, indicating it will often find the global optimum when terminated early.

\begin{figure}[!t]
\centering
\begin{subfigure}[]{0.34\columnwidth}
\includegraphics[trim=0pt 0pt 0pt 10pt, clip=true, scale=.235]{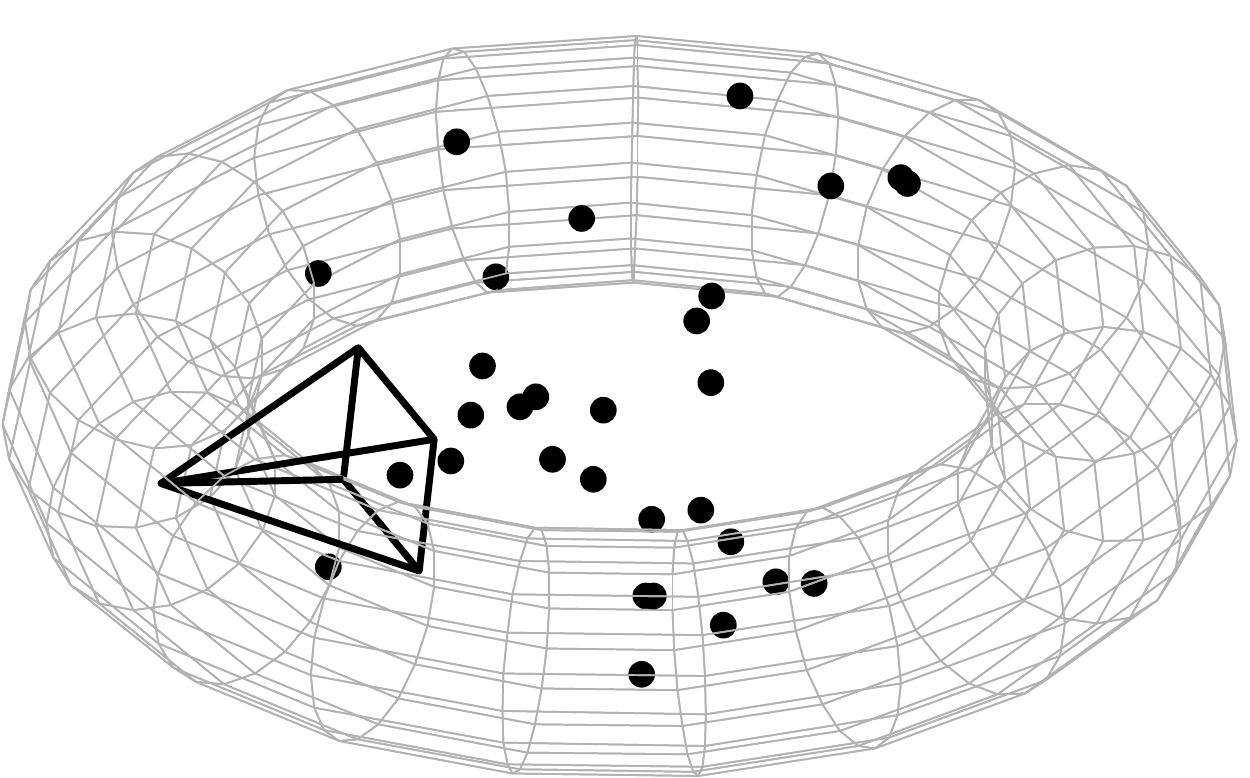}
\vfill
\includegraphics[trim=0pt 0pt 0pt 0pt, clip=true, scale=.235]{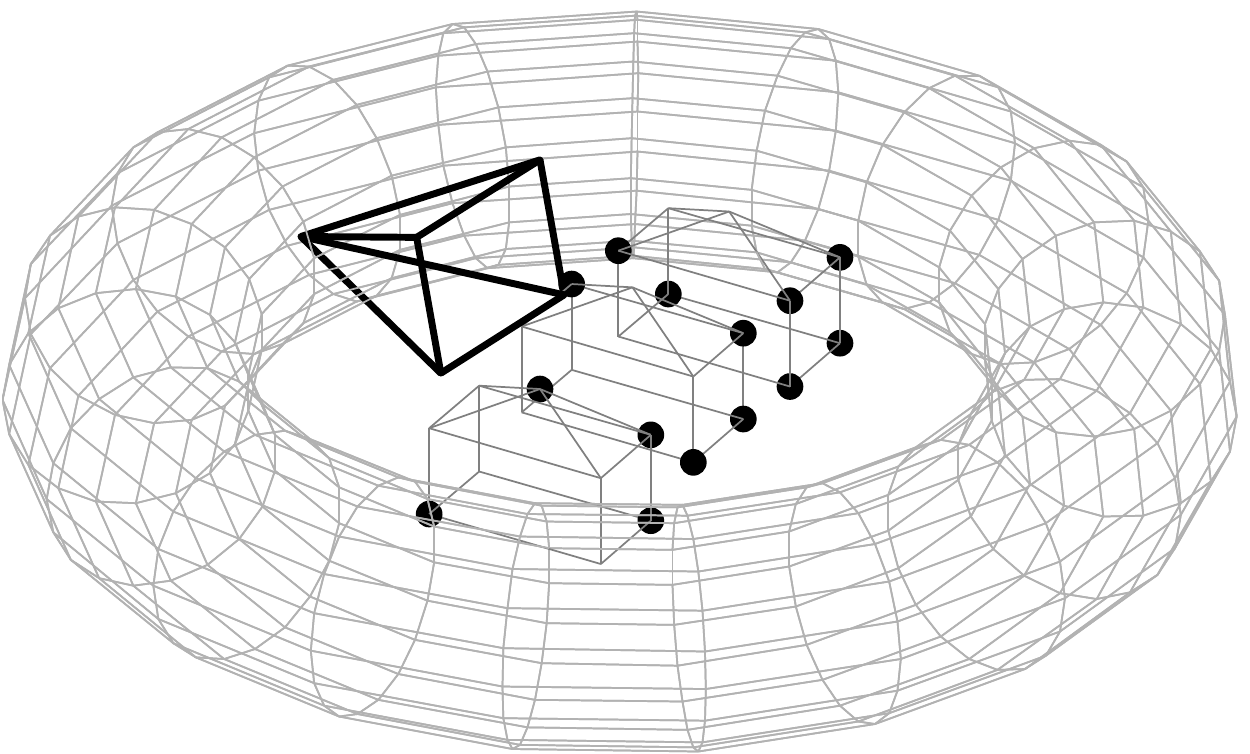}
\caption{3D Models}
\label{fig:models}
\end{subfigure}
\hfill
\begin{subfigure}[]{0.28\columnwidth}
\includegraphics[trim=2pt 6pt 4.5pt 8.5pt, clip=true, scale=.58]{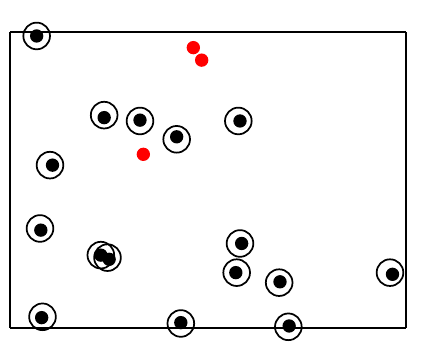}
\vfill
\includegraphics[trim=2pt 0pt 4.5pt 2pt, clip=true, scale=.58]{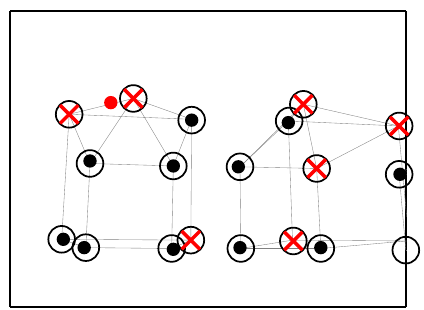}
\caption{2D Alignment}
\label{fig:2dalignment}
\end{subfigure}
\hfill
\begin{subfigure}[]{0.33\columnwidth}
\includegraphics[trim=0pt 0pt 0pt 2pt, clip=true, scale=.58]{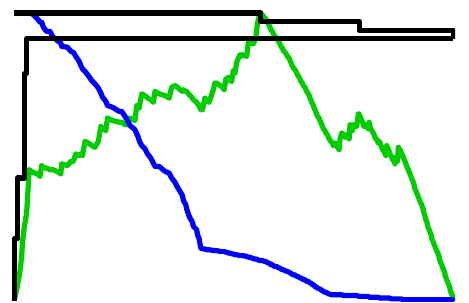}
\vfill
\includegraphics[trim=0pt 0pt 0pt 0pt, clip=true, scale=.58]{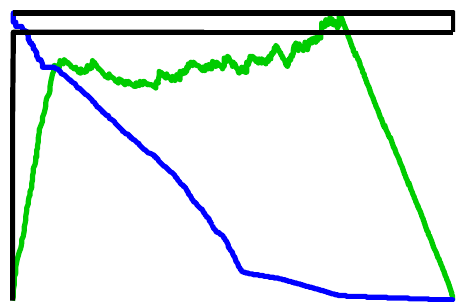}
\caption{Bound Evolution}
\label{fig:bound_evolution}
\end{subfigure}
\caption{Sample 2D and 3D results for two trials using the random points and repetitive CAD model datasets. (\subref{fig:models}) 3D models, true and GOPAC-estimated camera fulcra (completely overlapping) and toroidal pose priors. Only non-occluded 3D points are shown. (\subref{fig:2dalignment}) True projections of non-occluded 3D points are shown as black dots, 2D outliers as red dots, GOPAC projections as black circles and GOPAC-classified 3D outliers as red crosses. (\subref{fig:bound_evolution}) Evolution over time of the upper and lower bounds (black), remaining translation volume (blue) and translation queue size (green) as a fraction of their maximum values. Best viewed in colour.}
\label{fig:models_sample_results_bound_evolution}
\end{figure}

To show the improvement attributable to the tighter upper bounds derived, we measured the runtime of the algorithm with 10 random 3D points and 50\% 2D outliers using different upper bounds, shown in Figure \ref{fig:upper_bound_comparison}. The weak sphere-based bounding functions in (\ref{eqn:rotation_uncertainty_angle_weak}) and (\ref{eqn:translation_uncertainty_angle_weak}) are denoted $\psi_{r}^{w}$ and $\psi_{t}^{w}$ respectively, the tighter cuboid-based bounding functions in (\ref{eqn:rotation_uncertainty_angle}) and (\ref{eqn:translation_uncertainty_angle}) are denoted $\psi_{r}$ and $\psi_{t}$ respectively and the bounding function from (\ref{eqn:upper_bound_tighter}) is denoted $\Gamma$. Further results are provided in the appendix.

\begin{figure}[!t]
\centering
\includegraphics[trim=3pt 5pt 0pt 3pt, clip=true, width=0.75\columnwidth]{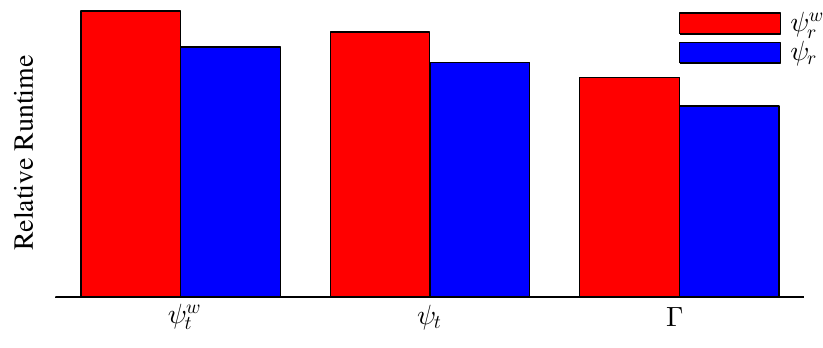}
\caption{Comparison of the different upper bound functions. Runtime is plotted relative to the maximum (leftmost) value. The weakest upper bound is 50\% slower than the tightest upper bound.}
\label{fig:upper_bound_comparison}
\vspace{-12pt}
\end{figure}

\subsection{Real Data Experiments}
\label{sec:results_real}

To evaluate the algorithm on real data, we use the \textsc{Data61/2D3D} (formerly NICTA) dataset \cite{namin2015multi}, a large and repetitive multi-modal outdoor dataset. Finding the pose of a camera within a large laser-scanned point-set without a good initialisation represents an unsolved problem in computer vision, which this work makes progress towards solving. For each image, we obtain the ground truth camera pose from the provided 2D--3D correspondences using EP$n$P \cite{lepetit2009epnp} followed by nonlinear P$n$P \cite{kneip2014opengv}. Extracting points from a laser scan that correspond to known pixels in an image is itself a challenging unsolved problem for 2D--3D registration pipelines. Due to the robust and optimal nature of GOPAC, we can relax this problem to isolating regions of the point-set that appear in the image and vice versa, from which putative correspondences may be drawn. We used semantic segmentations of the images and point-set to select regions that were potentially observable in both modalities, in this case the `building' class. We then used grid downsampling and $k$-means clustering on the class pixels and points independently to reduce them to a manageable size and converted the pixels to bearing vectors. While we do not know the correspondences in advance, each bearing vector has a good chance of having a 3D point as an inlier. In this way, we constructed a dataset consisting of a 3D point-set with 88 points, a set of 11 images containing 30 2D features and a set of ground truth camera poses. For this experiment, we used an inlier threshold of $\theta = 2^{\circ}$, multithreading and a 2D outlier fraction guess of $\omega_{\text{2D}} = 0.25$. The translation domain was $50\times5\times5$m, covering two lanes of the road, making use of the knowledge that the camera was mounted on a survey vehicle. SoftPOSIT and BlindPnP failed to find the correct camera pose for every image in this dataset, even when supplied the ground truth pose as a prior, due to the weak ground truth correspondences and an inability to handle 3D points behind the camera. Moreover, they do not natively support panoramic imagery and required an artificially restricted field of view to function.

Qualitative results for the GOPAC and RANSAC algorithms are shown in Figure~\ref{fig:results_2d3d} and quantitative results in Table~\ref{tab:results_2d3d}. GOPAC finds the optimal number of inliers for all frames and the correct camera pose for the majority of frames, despite the weakness of the 2D/3D point extraction process, surpassing the other methods. The failure modes for GOPAC were $180^{\circ}$ rotation flips, due to ambiguities arising from the low angular separation of points in the vertical direction. The difficulty of this ill-posed problem is illustrated by the performance of truncated GOPAC, which was not able to find all optima even after running for 30s, motivating the necessity for globally-optimal guided search.

\begin{table}[!t]
\centering
\caption{Camera pose results for the \textsc{Data61/2D3D} dataset. The median translation error, rotation error and runtime and the mean inlier recall and success rates are reported. $\left\lfloor\text{GP}\right\rfloor$ denotes truncated GOPAC, where search is terminated after $30$s, with no optimality guarantee. $\text{RS}_{K}$ denotes RANSAC with $K$ million iterations.}
\label{tab:results_2d3d}
\newcolumntype{C}{>{\centering\arraybackslash}X}
\begin{tabularx}{\columnwidth}{l C C C C}
\hline
Method & GP & $\left\lfloor\text{GP}\right\rfloor$ & $\text{RS}_{20}$ & $\text{RS}_{280}$\\
\hline
Translation Error (m)		& \textbf{2.30} & 3.10 & 20.3 & 28.5\\
Rotation Error ($^{\circ}$) & \textbf{2.08} & 3.04 & 178  & 179\\
Recall (Inliers)				& \textbf{1.00} & 0.97 & 0.75 & 0.81\\
Success Rate (Inliers)		& \textbf{1.00} & 0.45 & 0.00 & 0.00\\
Success Rate (Pose) 		& \textbf{0.82} & 0.64 & 0.09 & 0.09\\
\hline
Runtime (s) 				& 477 			& 34   & 34	  & 471\\
\hline
\end{tabularx}
\vspace{-12pt}
\end{table}

%%%%%%%%%%%%%%%%%%%%%%%%%%%%%%%%%%%%%%%%%%%%%%%%%%%%%%%%%%%%%%%%%%%%%%%%%%%%%%%%%%%%%%%%%%%%%%%
\section{Conclusion}
\label{sec:conclusion}

In this paper, we have introduced a robust and globally-optimal solution to the simultaneous camera pose and correspondence problem using inlier set cardinality maximisation. The method applies the branch-and-bound paradigm to guarantee optimality regardless of initialisation and uses local optimisation to accelerate convergence. The pivotal contribution is the derivation of the function bounds using the geometry of $SE(3)$. The algorithm outperformed other local and global methods on challenging synthetic and real datasets, finding the global optimum reliably. Further investigation is warranted to develop a complete 2D--3D pipeline, from segmentation and clustering to alignment.

{\small
\bibliographystyle{ieee}
\bibliography{citations_abbrv}
}

\end{document}